\documentclass[12pt]{elsarticle}
\usepackage{amssymb, amsmath, amsthm, bm}
\usepackage[noend]{algpseudocode}
\usepackage{algorithmicx,algorithm}
\usepackage{graphicx}
\usepackage{caption}
\usepackage[colorlinks,linkcolor=blue]{hyperref}
\usepackage{epstopdf}

\newtheorem{thrm}{Theorem}
\newtheorem{lmm}{Lemma}
\newtheorem{rmrk}{Remark}
\numberwithin{equation}{section}

\journal{}

\begin{document}
	
\begin{frontmatter}
		
\title{Improve Adversarial Robustness via Weight Penalization on Classification Layer}
		
\author[YantaiUniv1]{Cong Xu}
\author[S]{Dan Li}
\author[YantaiUniv2]{Min Yang}
\address[YantaiUniv1]{Department of Mathematics,
Yantai University,
Yantai
Shandong, China
congxueric@gmail.com
}
\address[S]{Research Center for Brain-inspired Intelligence and National Laboratory of Pattern Recognition,
Institute of Automation Chinese Academy of Sciences,
Beijing,China
lidan2017@ia.ac.cn}
\address[YantaiUniv2]{Department of Mathematics,
Yantai University,
Yantai
Shandong, China
yang@ytu.edu.cn
}

\begin{abstract}
It is well-known that deep neural networks are vulnerable to adversarial attacks.
Recent studies show that well-designed classification parts can lead to better robustness.
However,  there is still much space for improvement along this line.
In this paper, we first prove that,  from a geometric point of view,
the robustness of a neural network is equivalent to some angular margin condition of the classifier weights.
We then explain why ReLU type function is not a good choice for activation under this framework.
These findings reveal the limitations of the existing approaches and
lead us to develop a novel light-weight-penalized defensive method, which is simple and has a good scalability.
Empirical results on multiple benchmark datasets  demonstrate that our method can effectively improve the robustness of the network
without requiring too much additional computation, while maintaining a high classification precision for clean data.
\end{abstract}

		
\end{frontmatter}	

\section{Introduction}
Although deep neural network (DNN) has achieved state-of-the-art performance on many challenging computer vision tasks
\cite{hu2019, huang2019, krizhevsky2009,lecun1998,nair2010,wideresnet2016,zhang2019c},
it was found that they are susceptible  to some crafted adversarial samples \cite{goodfellow2015, szegedy2013}:
very small and often imperceptible perturbations of the data samples are sufficient to fool state-of-the-art classifiers
and result in incorrect classification.

This phenomenon has received particular attention.
Recently, there has been a great deal of interest in developing various defense mechanisms to enhance the robustness of neural networks.
Examples include defensive distillation \cite{carlini2017, papernot2016b},
feature squeezing \cite{xu2018},
obfuscated gradients \cite{guo2018, na2018, samngouei2018},
adversarial training \cite{madry2018, shafahi2019, shaham2018,wang2020},
and several other versions  \cite{guo2019, naveiro2019, pang2018, wan2018}.
However, some powerful defensive methods such as adversarial training  \cite{madry2018, wang2020}
suffer from heavy computational cost and are nearly intractable on large scale problems.
Besides, a large number of them improve robustness at the expense of the accuracy of clean data \cite{su2018, zhang2019b}.
Recent studies \cite{mustafa2019, pang2018, pang2020}  show that well-designed classifiers can effectively learn more structured representations,
thus greatly improving the robustness of neural networks.
However, there is still  much room for improvement along this line.

With this paper, we aim to propose a novel defensive method
by designing a special weight penalization mechanism for the classification part.
We first argue that, from a geometric point of view,
the robustness of a neural network  can be depicted by certain angular margin condition on classifier weights  (see Section 2.2).
Increasing the angular margins between different weight vectors is helpful to  raise the robustness of  the neural network.
Secondly, we point out that,  in terms of classification robustness,
the commonly used ReLU type function is not a good activation choice in a framework like ours.
Based on the above findings,
a special light-weight-regularization term is proposed to the commonly used loss function to advance both the robustness and the accuracy of the models.
Extensive experiments on three datasets show that our method can achieve good robustness under some typical  attacks
while maintaining high accuracy to clean data.

The main contributions of this work are summarized as follows:
\begin{itemize}
\item [(1)]
  We prove that the robustness of the neural networks is equivalent to certain angular margin condition
  between the weight vectors of the classifier.
  We further point out that in a framework like ours,
  a saturated function such as Tanh is a more appropriate activation option than the commonly used  ReLU function.
\item [(2)]
  We design a novel light penalty defense mechanism that only needs to control the weight vectors at the classification part
  to improve the robustness of the neural network.
  The proposed method is simple and could be easily  used in various neural network  architectures.
\item [(3)]
  Empirical  performances on MNIST, CIFAR-10, and CIFAR-100 datasets show that the proposed approach
  has reliable robustness against different adversarial attacks in different networks,
  while maintaining high classification accuracy for clean inputs.
  In most cases, our method outperforms the compared benchmark defense approaches.
\end{itemize}

The rest of the paper is organized as follows.
Section 2 introduces the related works on defensive approaches and attack models.
Our method is  formally  launched and introduced in Section 3.
Experimental results under different threat models, as well as comparisons to other defense methods, are presented in Section 4.
The conclusion is drawn in Section 5.
Details of the theoretical analysis are left in the appendix.

\section{Related Works}

\subsection{ Defense methods}
Previous efforts to improve the robustness of deep learning models are encouraging \cite{athalye2018, guo2018, na2018,madry2018, pang2018,foolbox2017,samngouei2018,wang2018}.
Among them, adversarial training \cite{madry2018} is one of the most effective means of defense.
Adversarial training involves putting the adversary in the training loop,
and perturbing each training sample before it is passed through the model.
One of the main drawbacks of adversarial training  is  its  high computational cost,
each stochastic gradient descent iteration using multiple gradient computations to generate adversarial samples,
which limits its application in large-scale problems.
Although there appears a number of fast procedures \cite{shafahi2019, zhang2019a}, this problem still exists.
Another problem with adversarial training is that the improved robustness comes at the expense of the prediction accuracy of clean inputs \cite{su2018, zhang2019b}.

To save the computational cost and ensure predictive accuracy on clean data,
some recent studies aim to improve the linear classification part of the neural network to produce reliable robustness.
Wan \& Zhong et al. \cite{wan2018} and Wang et al. \cite{wang2018} have designed some regularization items in the classifier loss function
to increase the dispersion between classes of feature representation.
Recently, Pang et al. \cite{pang2018, pang2020} further present some novel  loss functions based on Max Mahalanobis distribution
to learn more structured representations and to induce high-density regions in feature space.
Although our approach  shares some similar spirit as the work of \cite{pang2018, pang2020},
there exist  two critical differences:
\begin{itemize}
\item
In  \cite{pang2018, pang2020}, the weights on the classification layer were explicitly defined to meet the angular margin condition.
This explicit definition may limit the learning ability of the networks,
especially when the dimension of the feature representations are greater than the number of the neurons in the classification layer.
On the contrary, our method implicitly treats the weights restriction as a regular term in the optimization objective.
Such modification makes the networks more flexible and have better learning ability.
\item
ReLU function is a popular activation choice in deep learning neural networks, which was also used in \cite{pang2018, pang2020}.
However, according to our analysis in Section 3.2, we find that ReLU function is not a good  choice in a framework like ours.
Instead,  using saturate activation functions like Tanh  may improve the robustness of the model.
\end{itemize}

\subsection{ Attack models}
A large amount of attack models and algorithms have been introduced in recent years \cite{ carlini2017, goodfellow2015,  moosavi2016, madry2018, foolbox2017, szegedy2013}.
The purpose of these attack models is to look for a small and undetectable perturbation to add to the input, thus generating adversarial examples.
In this article, we mainly consider the following four attack approaches.

\textbf{DeepFool} \cite{moosavi2016} aims to find the minimum adversarial perturbations on both an affine binary classifier and a general binary differentiable classifier.

\textbf{Carlini \& Wagner Attack} (\textbf{C\&W}) \cite{carlini2017} is a widely used attack method which applies
a binary search mechanism on its hyperparameters to find the minimal $ l_2 $ distortion for a successful attack.

\textbf{Projected Gradient Descent} (\textbf{PGD}) \cite{madry2018} starts from a random position in the clean image neighborhood
and applies FGSM \cite{goodfellow2015} for $ m $ iterations with a step size of $ \alpha $.
More specifically, for a given iterative number $ m $ and a small adversarial manipulation budget  $\epsilon $,
\begin{align*}
  &\hat{x}_i = x_{i-1} + \alpha \cdot \mathrm{sign} (\nabla_{x_{i-1}} \mathcal{L}(x_{i-1}, y)),
    \\[5pt]
 & x_{i} = \mathrm{clip} (\hat{x}_i, x - \epsilon, x + \epsilon),
 \quad
  i=1,\ldots, m,
\end{align*}
where  $ \mathcal{L} $  denotes the classification loss  and $ \mathrm{clip} (\cdot) $ projects $\hat{x}_i$ into the $l_{\infty}$ ball around $x$ with radius $\epsilon$.
PGD is one of the most powerful iterative attacks that relies on the first-order information of the target model.

\textbf{Multi-Target Attack} (\textbf{MTA}) \cite{tramer2020} is a recently proposed adaptive attack,
 which directly minimizes the Max Mahalanobis center loss for each target class.

\section{Robust Neural Network via Weight Penalization}
A typical feed forward neural network consists of a nonlinear feature extraction part that transforms the input $x $ to certain feature representation $ f(x) $:
\begin{align*}
x \rightarrow f(x) := \phi(h(x)),
\end{align*}
where $ \phi $ denotes the activation function applied to the last layer of the encoder,
and $ h(x) $ is the feature representation before activation.

With $ f(x) $, we can define the following linear classifier:
\begin{align}
\label{classifier}
c(x) = W f(x) + b,
\end{align}
where $ W $ and $ b $ are the weight matrix and bias vector of the classification layer, respectively.
Then one can use $c(x) $ to predict the category of the input $ x $.

When only the accuracy of classification is considered,
neural networks only need to learn the features that are linearly separable.
However, it was soon realized that these features are not robust under adversarial attack.
Given a clean sample $ x $, an adversarial attack attempts to find a crafted sample $x'$ such that
\begin{align*}
\hat{y}(x') \not=\hat{ y}(x), \quad \textrm{s.t.} \quad \|x' - x\| \le \epsilon,
\end{align*}
where $\hat{y}(\cdot)$ denotes the prediction label from the classifier and $ \| \cdot \| $ is some vector norm.

\subsection{Robustness from a geometric point view}
Inspired by  the analysis in \cite{wang2018},
in this section we aim to relate the robustness of neural networks  to the angular margin between the classifier weights.
The  bias  $ b $ in \eqref{classifier}  is geometrically less informative,
so  we assume that $ b $ is zero for brevity.

\begin{figure}[!htb]
	
	\centering
	\begin{minipage}{0.4\textwidth}
		\centering
		\scalebox{0.42}{\includegraphics{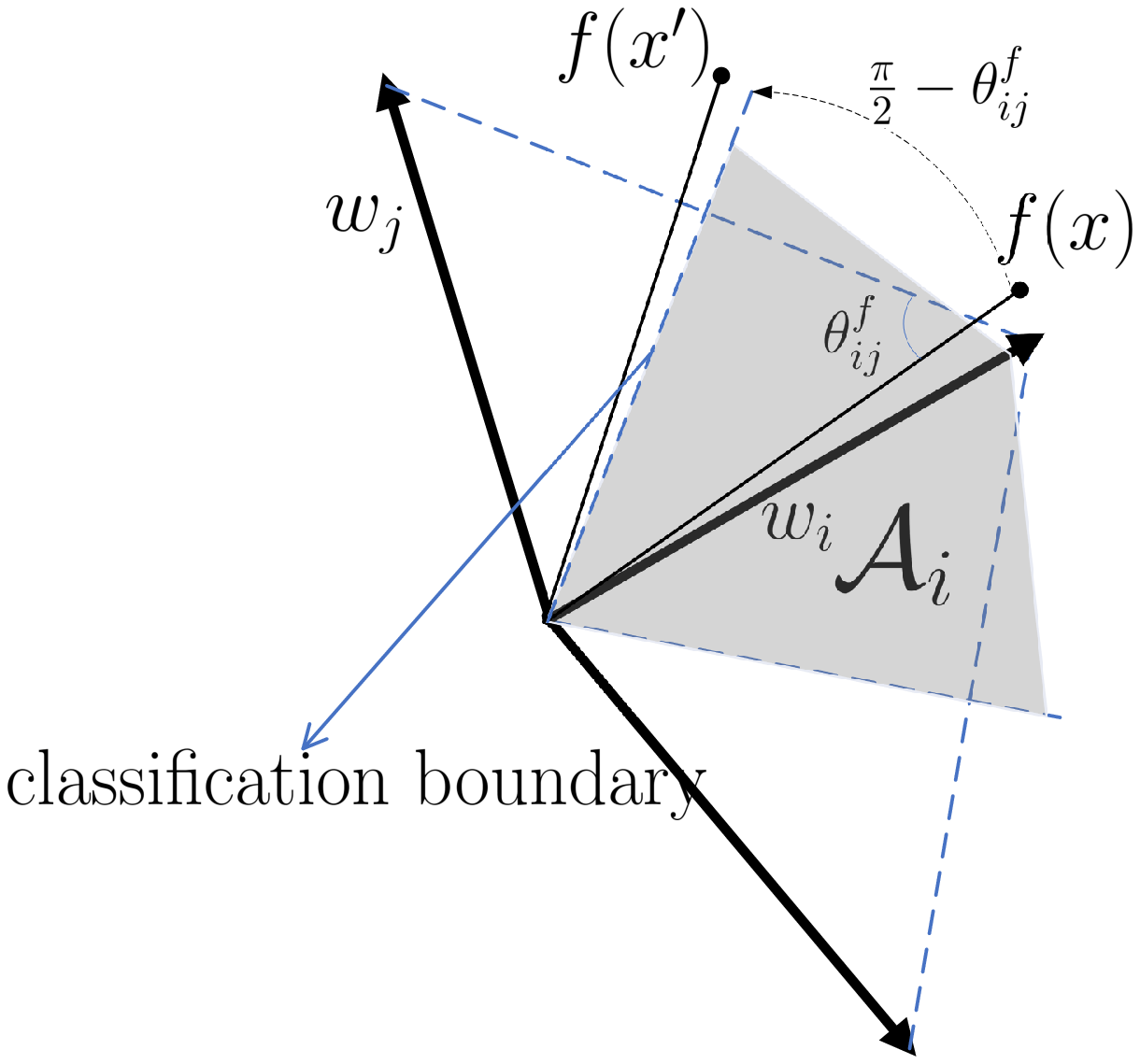}}
		\caption*{(a)}
	\end{minipage}
	\begin{minipage}{0.4\textwidth}
		\centering
		\scalebox{0.36}{\includegraphics{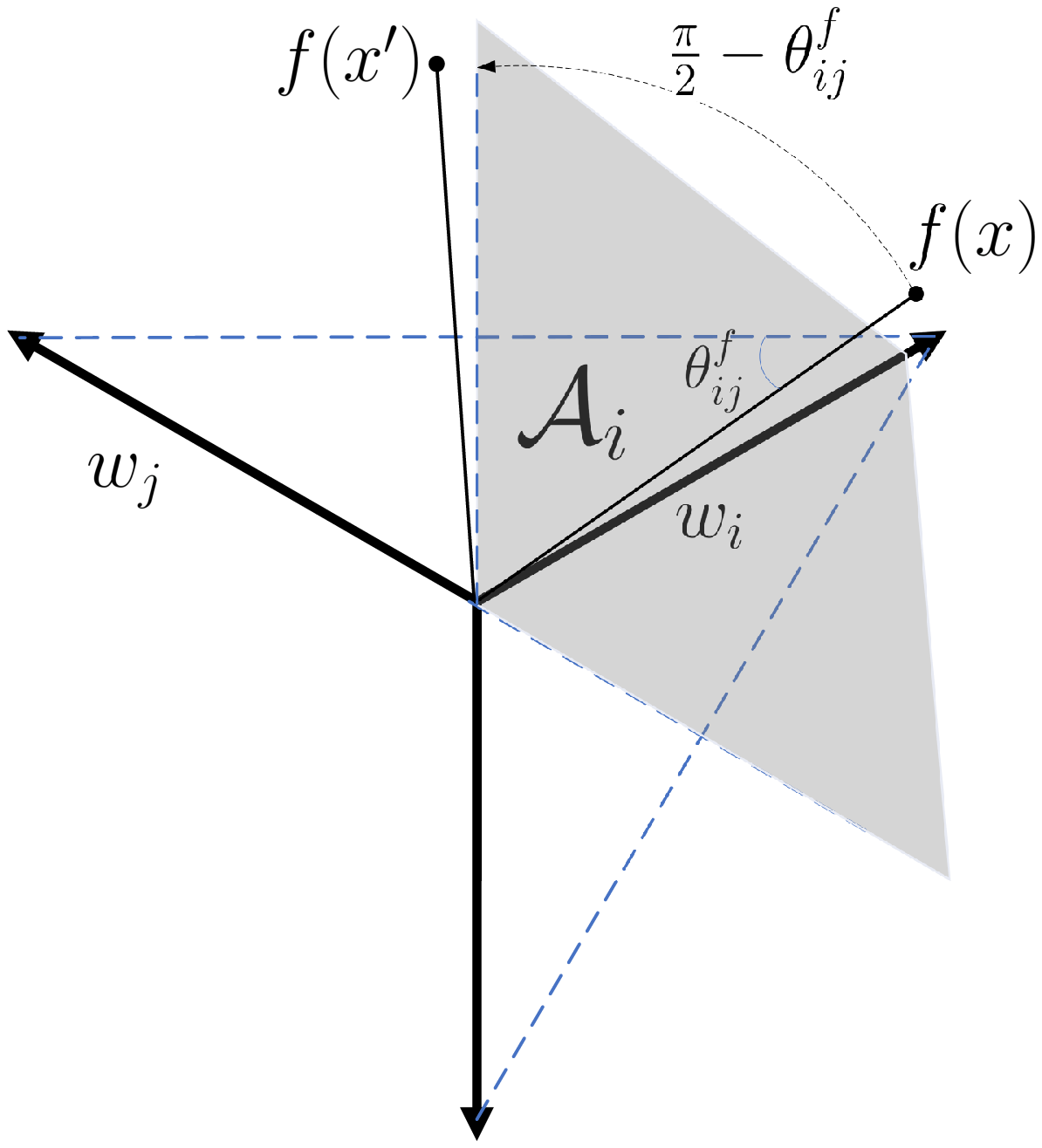}}
		\caption*{(b)}
	\end{minipage}
		\caption{From the geometric point of view, pattern (b) is more robust than pattern (a).}
	\label{geometric}
\end{figure}

Assume that the feature representation $ f(x) \in \mathbb{R}^P$ and weight matrix $ W \in \mathbb{R}^{K \times P} $ ,
where $ P $ is the dimension of the learned features and $ K $ is the number of categories.
Define the classification domains as follows
\begin{align}
\label{region}
\mathcal{A}_i :=\{f \in \mathbb{R}^P: (w_i - w_j)^T f \ge 0, j \not = i \},
\quad  1\leq i \leq K,
\end{align}
where $ w_i $, $ 1\leq i \leq K$,  is the $i $-th  row of $W$.
If the learned feature of an input belongs to domain $ A_i $,
then we predict it to  the class $ i $ .

Suppose that there exists  an  input $ x $  whose feature representation $ f(x) $ belongs to $ \mathcal{A}_i $  (see e.g. Figure \ref{geometric}).
Note that
\begin{align}
\label{boundary}
(w_i - w_j)^T f(x) = \|w_i - w_j\|_2 \|f(x)\|_2 \cos \theta_{ij}^{f},
\quad
j \neq i,
\end{align}
where $\theta_{ij}^{f}$ denotes the angle between $ f(x) $ and $w_i - w_j$.
From \eqref{region} and \eqref{boundary},
to make  $ f(x') $  belong to $\mathcal{A}_j$, where  $ \|x' - x\| \le \epsilon $,
a successful attacker needs to rotate $ f(x) $  at least   $\frac{\pi}{2} - \theta_{ij}^{f}$ degrees.
As a consequence,
for a fixed perturbation amplitude,
the smaller the angle $\theta_{ij}$, the harder the attack.
In other words, the smaller
\begin{align}
\label{theta}
   \max_{1 \leq i,j \leq K \atop j \neq i} \{\theta_{ij}^{f}\},
\end{align}
the more robust the neural network will be.

Further, for a well-trained robust neural networks,
thanks to  the universal approximation property \cite{kidger2019, pinkus1999},
we can expect that all data features are concentrated in the middle of the classification areas.
Therefore, we have $ \theta_{ij}^{f} \approx \frac{\pi}{2}-\frac{\angle (w_i, w_j)}{2} $.
Combining this approximation with \eqref{theta}, we  deduce that  the larger
\begin{align}
\label{ww}
   \min_{1 \leq i,j \leq K \atop j \neq i}\frac{ \angle (w_i, w_j)}{2}
\end{align}
the more robust the neural network will be.
Nevertheless, this conclusion is not  applicable  directly  to optimization.
In the following discussion, we will present some equivalent theoretical results
that  can be directly applied to the optimization objective.

\begin{thrm}\label{margin}
 Let  $ W \in \mathbb{R}^{K \times P } $ be  the weight matrix of the classification layer,
 where $ K $ denotes the number of categories and $ P $ is the dimension of the learned features.
 Assume that $ 1<K \leq P+1 $.
 For simplicity, we  assume that all column vectors of $ W $ have the same length $ s $.
 Then \eqref{ww} achieves the maximum if and only if  $  W $ satisfies
	\begin{align}
   \label{innercondition}
	w_i^T w_j = \frac{s^2}{1-K}, \quad \forall 1\leq i,j \leq K, \; j \neq i.
	\end{align}
\end{thrm}
\begin{proof}
	See Appendix.
\end{proof}

\begin{rmrk}
    The assumption  that $ 1<K \leq P+1 $ almost always holds true in computational vision tasks.
    Note that there is a similar theoretical result in the literature \cite{pang2018}.
    But an additional  assumption $\sum_{i=1}^K w_i=\bm{0}$ is needed there.
\end{rmrk}

\begin{rmrk}
According to Theorem \ref{margin}, to obtain a robust neural network,
the weight vectors of the classification layer must be uniformly distributed in the latent feature space (see e.g. Figure \ref{geometric} (b)).
\end{rmrk}

It is trivial that the condition \eqref{innercondition}  has an equivalent matrix form
\begin{align}
\label{matrixform}
  W  W ^T =\Sigma
\end{align}
where the entries of the matrix $ \Sigma $ satisfies
\begin{align*}
\Sigma_{ij} =
\left \{
\begin{array}{ll}
s^2, & i=j \\
\frac{s^2}{1-K}, & i \not = j
\end{array} \right..
\end{align*}
Now we  impose the constraint \eqref{matrixform} into the loss function and
obtain the following complete optimization objective:
\begin{align}
\label{weightloss}
\mathcal{L} := \mathcal{L}_{SCE} + \alpha  \|W  W^T - \Sigma\|_2,
\end{align}
where $\mathcal{L}_{SCE}$ represents the commonly used softmax cross-entropy loss,
$  \alpha $ and $ s $ are hyperparameters.

It is obvious that the proposed method  has a clear geometric background and is easy to implement,
only an additional weight penalization on the classification layer requiring to be considered.

\subsection{ReLU type function is not a good choice for activation}
Non-saturated functions such as ReLU \cite{nair2010} and PReLU \cite{he2015} are commonly used in activation units
due to their simpleness and effectiveness.
Nonetheless,  if these activation functions are used in a framework like ours, two problems may arise.

First, consider the example shown in Figure \ref{activation},
where the category number $ K=3 $ and the feature dimension $ P=2 $.
Although the weight vectors have been evenly distributed in the feature space,
we find that $w_3 $ belongs to the third quadrant,
and therefore both of its components are negative.
Notice that the value of the ReLU activation function is always positive,
which leads to the fact $ w_3^T f(x) \le 0$ regardless of the label of the input.
According to  \eqref{region}, it implies that all inputs belong to the class 3 will be misclassified.
Therefore, the use of ReLU type activation function might deteriorate the classification accuracy of the model.
Such case is very likely to happen if  the weights are fixed in advance as done in \cite{pang2018}.
Instead, if  we adopt an activation function  like Tanh, which belongs to $ (-1,1) $,
 the above phenomena can be avoided.

\begin{figure}[!htb]
	
	\centering
 	\scalebox{0.35}{\includegraphics{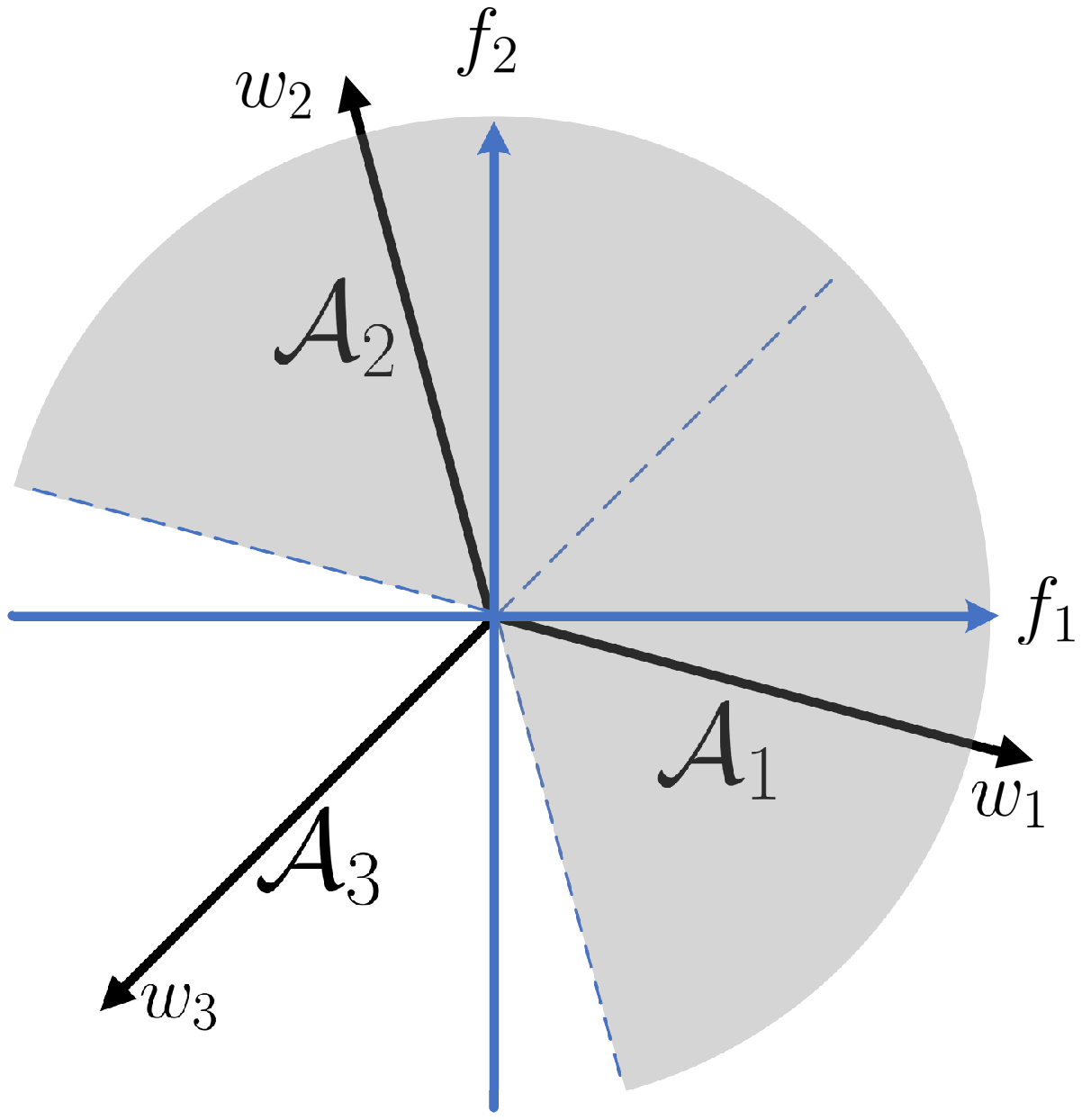}}
 	\caption{An example using ReLU activation function fails to give the correct classification.}
	      \label{activation}
\end{figure}

Secondly, suppose that a neural network has already learned linearly separable features after several epoches.
Given a training sample $ (x,y) $, where $y \in \{1,2,\ldots,K\} $,
it follows from \eqref{region} that $ (w_y- w_j)^T f(x) >0$ , $ \forall j \not =y $.
Note that
\begin{align*}
      \mathcal{L}_{SCE}(x, y)
& =-\log \frac{\exp(w_y^T f(x))}{\sum_{j=1}^K \exp(w_j^T f(x))} \\
& =-\log \frac{1}{\sum_{j=1}^K \exp(- f(x)^T(w_y-w_j))}.
\end{align*}
Then,
\begin{align*}
\mathop{\lim} \limits_{0<\beta \rightarrow +\infty} -\log \frac{1}{\sum_{j=1}^K \exp(-\beta f(x)^T(w_y-w_j) )}
   = 0.
\end{align*}
Since the value of ReLU type function belongs to $[0, +\infty) $,
we can infer from the above two equations that if  a ReLU type activation is used,
after the neural network is trained to a certain extent,
it does not need to  pay attention to the second penalty term in \eqref{weightloss},
but only needs to simply  prolong the  length of feature vector,  to reduce the total loss.
However, according to the analysis in Section 3.1,
prolonging the  length of features is not helpful to improve the robustness of the neural network,
because the cost of rotating a feature vector from one region to another is unchanged.
On the contrary, if a saturated activation like Tanh is used,
the length of  feature vector is  bounded in the interval $ (0,\sqrt{P}) $.
So, in order to reduce the total  training loss,
the neural network must take care of  the penalty term in \eqref{weightloss} and thus could learn more robust features.

\begin{figure}[!htb]
	
	\centering
	\begin{minipage}{0.3\textwidth}
		\centering
		\scalebox{0.7}{\includegraphics{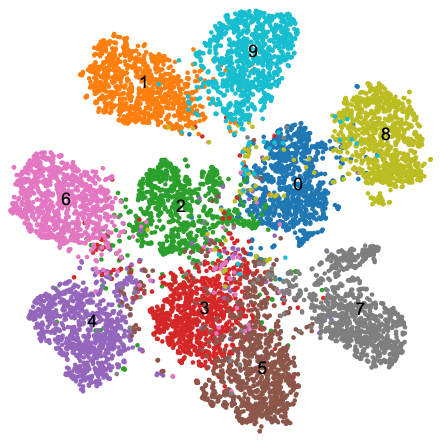}}
		\caption*{(a) Ordinary softmax loss}
	\end{minipage}
	\begin{minipage}{0.3\textwidth}
		\centering
		\scalebox{0.7}{\includegraphics{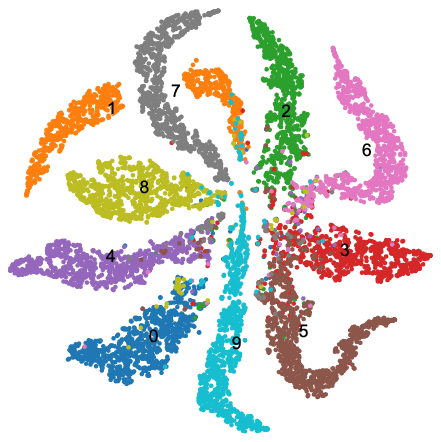}}
		\caption*{(b) Weight penalty +ReLU }
	\end{minipage}
	\begin{minipage}{0.3\textwidth}
		\centering
		\scalebox{0.7}{\includegraphics{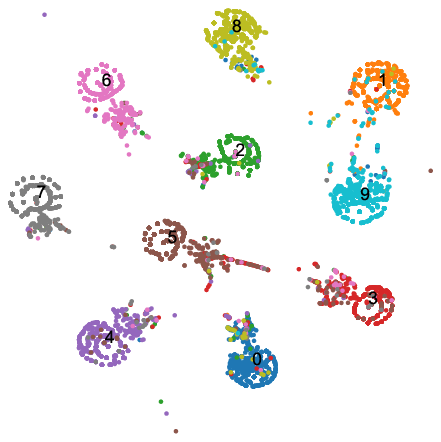}}
		\caption*{(c) Weight penalty+Tanh}
	\end{minipage}
	\caption{t-SNE \cite{maaten2008} visualization of latent features on the test data of CIFAR-10, where the index numbers indicate classes 0 to 9.}
	\label{features}
\end{figure}

According to the analysis above,
in this paper,  we  choose Tanh as the activation function in the last layer of the encoder part, i.e.,
\begin{align*}
f(x) = \mathrm{Tanh} (h(x)).
\end{align*}
We depict the features exacted based on ReLU and Tanh activations  in Figure \ref{features}, respectively.
Consistent with our analysis,  the features exacted based on Tanh activation are more compact
and  far from the decision boundary.
Therefore, an adversarial attack which only makes small changes to the input is less likely to succeed.

\section{Experiments}
In this section, we will test the robustness and classification accuracy of the proposed method
over different networks and different data sets.

The attack methods introduced in Section 2.2, including PGD \cite{madry2018}, DeepFool \cite{moosavi2016},
C\&W \cite{carlini2017} and the multi-target attack (MTA) \cite{ilyas2019},
are applied in our experiments.
Among them, PGD, DeepFool, MTA are  iterative-based methods,
which usually iterate less than hundreds of rounds to craft an adversarial example.
On the other hand, C\&W is an optimization-based method,
which require a much more computation compared to the iterative-based methods,
but usually has higher success rates on attacking classifiers.

According to the suggestions of the literature,
Table \ref{tableA} lists the hyperparameters of these attack methods, including the number of iterations, step size and overshoot.
 We shall use PGD20 to denote a PGD attack with 20 iteration steps.
In the following text, similar symbols indicate similar meanings.

\begin{table}[!htb]
	\center
	\caption{The basic setup for four adversarial attacks}
	\label{tableA}
	\scalebox{0.8}{
		\begin{tabular}{|c|c|c|c|c|c|c|c|c|}
			\hline
			&   \multicolumn{4}{|c|}{PGD}        & C\&W  & \multicolumn{2}{|c|}{MTA}  & DeepFool \\
			\hline
			\hline
			number of iterations & 20 & 40 & 100 & 200  &  1000 & 100 & 200 &  50 \\
			\hline
			step size & 0.0031 & 0.01 & 0.0031 & 0.01 & 0.01 & 0.0031 & 0.01 &0.02(overshoot) \\		
			\hline
	\end{tabular}}
\end{table}

We consider three popular benchmark datasets,
including MNIST \cite{lecun1998}, CIFAR-10 and CIFAR-100 \cite{krizhevsky2009}.
Following the suggestions in \cite{carlini2019},
for PGD, DeepFool and MTA adversarial attacks,
we consider $l_{\infty}$ adversarial perturbations and set the perturbation amplitude $\epsilon=0.3$ for MNIST  and $\epsilon=8/255$ for CIFAR-10 and CIFAR-100.
For  C\&W attack,
we consider $l_2 $ perturbations and set $\epsilon=3 $ for MNIST
and $\epsilon=1 $ for CIFAR-10 and CIFAR-100.

We compare our method with four  baseline defensive algorithms,
including AdvTraining \cite{madry2018}, TRADES \cite{zhang2019b}, MART \cite{wang2020} and MMC \cite{pang2020}.
Among them, AdvTraining is regarded as one of the most effective defenses.
TRADES is a defensive method which keeps a good tradeoff between the adversarial robustness  and the prediction accuracy.
MART is based on AdvTraining but explicitly differentiates the misclassified and correctly classified examples during the training.
MMC is an improved version of the defense approach proposed in \cite{pang2018} and has a similar starting point to our approach.

If it is not stated otherwise, the hyperparameters in \eqref{weightloss} are fixed as $ \alpha=100 $ and $ s=0.1 $,
and hyperparameters of all compared methods are selected according to the original literature.

\subsection{Performance on large-scale networks}
In this section, we first compare our method with  the baseline defensive algorithms on CIFAR-10 dataset using large-scale networks.
All clean images are normalized and transformed by 4-pixel padding with $32 \times 32$ random crop and horizontal flip.

For our method, we use the same neural network architecture as  \cite{madry2018, zhang2019b},
i.e., the wide residual network WRN-34-10 \cite{wideresnet2016}.
We apply the Adam \cite{kingma2014} optimizer with an initial learning rate of 0.01 which decays with a factor of 0.9  for each 60 training epoches,
and run 400 epochs on the training dataset.
The experimental results of the compared approaches are directly cited from the literature.

\begin{table}[!htp]
	\center
	\caption{Classification accuracy on the clean and the adversarial examples on CIFAR-10.}
	\label{tableC}
	\scalebox{0.9}{
		\begin{tabular}{|c||c|c|c|c|}
\hline
Methods	                                &	Networks	&	Attacks	&	Clean accuracy	&	Robust accuracy	\\
\hline
AdvTraining \cite{madry2018} &	WRN-34-10	&	PGD20  	&	87.30\%	&	47.04\%	\\
MART \cite{wang2020}	            &	WRN-34-10	&	PGD20	    &	84.17\%	&	58.56\%	\\
MMC	 \cite{pang2020}              &	ResNet-32	    &	PGD10      &	92.70\%	&	36.00\%	\\
MMC  \cite{pang2020}    	        &	ResNet-32	    &	PGD50	    &	92.70\%	&	24.80\%	\\
TRADES \cite{zhang2019b}	    &	WRN-34-10	&	PGD20   	&	84.92\%	&	56.61\%	\\
TRADES \cite{zhang2019b}	    &	WRN-34-10	&	DeepFool	&	84.92\%	&	61.38\%	\\
TRADES \cite{zhang2019b}	    &	WRN-34-10	&	C\&W    	&	84.92\%	&	81.24\%	\\
\hline
Ours	            &	WRN-34-10	&	PGD20	    &	94.10\%	&	78.40\%	\\
Ours         	&	WRN-34-10	&	DeepFool	&	94.10\%	&	79.92\%	\\
Ours	            &	WRN-34-10	&	C\&W	    &	94.10\%	&	83.42\%	\\
			\hline
	\end{tabular}}
\end{table}
According to the results in Table \ref{tableC},
 we can see that the proposed defense method can significantly improve the robustness of the model,
 while still maintaining a high prediction accuracy for clean data.
 It outperforms the compared methods in both accuracy and robustness.

Next, we consider the CIFAR-100 dataset.
As we know,
CIFAR-100 is  a more complex dataset than CIFAR-10.
In Table \ref{tableD},
we report the prediction accuracy and robustness performance of our method on CIFA-100 dataset under various attacks.
Obviously, our method is still effective in this complex dataset.
\begin{table}[!htp]
	\center
	\caption{Classification accuracy of our method on CIFAR-100 dataset.}
	\label{tableD}
	\scalebox{0.9}{
	\begin{tabular}{|c|c|c|c|c|c|}
\hline
\multicolumn{6}{|c|}{CIFAR-100 } \\
\hline
\hline
Methods	&	Clean	&	PGD20	&	PGD100	&	DeepFool & C\&W 	\\		
			\hline
Ours     	&	73.84\%	&	47.17\% & 46.47\% & 55.21\% & 54.82\%	\\
			\hline
	\end{tabular}}
\end{table}

\subsection{Performance on small networks}
Most previous defensive algorithms were experimented with large-scale neural networks like ResNet \cite{he2016}.
However, under some resource constrained situations,
we have to turn to small networks due to their compute and power efficiency.
In this section, we are interested in the performance of the new method on small networks.

Specially, we consider a small network  \cite{carlini2017, papernot2016b} which only includes four convolutional layers.
We compare our method with AdvTraining  \cite{madry2018} and MMC \cite{pang2020} in this small network.
We apply the Adam \cite{kingma2014} optimizer with the learning rate of 0.01, and run for 400 epochs on the training samples.

\begin{table}[!htb]
	\center
	\caption{Classification accuracy on the clean and the adversarial examples on MNIST using a small network.}
	\label{tableMnist}
	\scalebox{0.9}{
		\begin{tabular}{|c|c|c|c|c|c|c|}
\hline
\multicolumn{7}{|c|}{MNIST } \\
\hline
\hline
Methods      &	Clean	&	PGD40	&	PGD200	&	DeepFool	&	C\&W	&		MTA200	\\
\hline
AdvTraining	&	\textbf{99.47\%}	&	\textbf{96.38\%}	&	\textbf{94.58\%}	&	96.79\%	&	39.82\%	&	\textbf{98.57\%}	\\
MMC	            &	89.55\%	            &	35.98\%	&	34.91\%	&	34.12\%	&	2.08\%	&		12.40\%	\\
Ours	            &	99.39\%	            &	81.90\%	&	80.39\%	&	\textbf{97.28\%}	&	\textbf{97.95\%}	&	88.31\%	\\
\hline
	\end{tabular}}
\end{table}

\begin{table}[!htb]
	\center
	\caption{Classification accuracy on the clean and the adversarial examples on CIFAR-10 using a small network.}
	\label{tableCifar}
	\scalebox{0.9}{
		\begin{tabular}{|c|c|c|c|c|c|c|}
\hline
\multicolumn{7}{|c|}{CIFAR-10  } \\
\hline
\hline
Methods &	Clean	&	PGD20	&	PGD100	&	DeepFool	&	C\&W	&		MTA100	\\
\hline
AdvTraining	&	68.81\%	&	21.36\%	&	18.94\%	&	24.68\%	&	10.91\%	&	57.15\%	\\
MMC	&	73.30\%	&	39.23\%	&	38.96\%	&	16.86\%	&	1.49\%	&	25.96\%	\\
Ours	&	\textbf{81.43\%}	&	\textbf{57.11\%}	&	\textbf{56.22\%}	&	\textbf{59.11\%}	&	\textbf{59.11\%}	&	\textbf{64.88\%}	\\
\hline
	\end{tabular}}
\end{table}

It is observed  from Tables \ref{tableMnist} and \ref{tableCifar}  that AdvTraining  has a relative better performance on simple MNIST dataset.
However,  on more complex CIFAR-10 dataset, the performance of AdvTraining declined significantly.
On the contrary, the performance of the other two methods is more consistent on both datasets.
We may infer from the above phenomenon that AdvTraining method is more dependent on the representative ability of the networks,
and a small network seems not sufficient to support it to learn useful features on complex datasets.
Moreover, it is worth to notice that MMC fails against C\&W1000 and MTA100 attacks,
while our method still exhibits excellent robustness.

\begin{figure}[!thb]
	\centering
	\scalebox{0.8}{\includegraphics{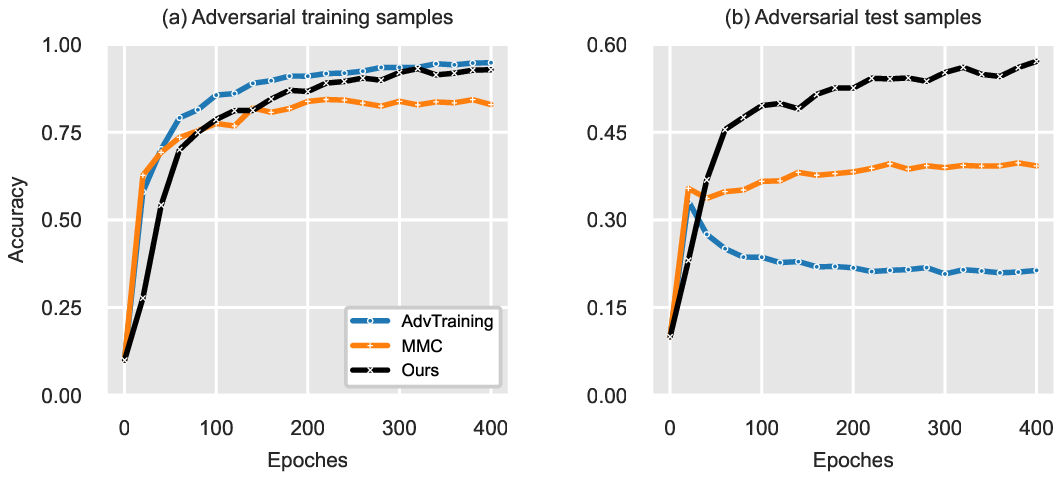}}
	\caption{Training and test accuracy of AT, MMC and our method with increasing epoches on CIFAR-10 dataset. The robustness is evaluated under PGD20 attack.}
	\label{small_cifar10}
\end{figure}

Next, we depict the  training curves of three defensive methods under PGD20 attack on CIFAR10 dataset.
As shown in Figure \ref{small_cifar10},
for AdvTraining,
the test accuracy under adversarial attacks has  a clear decline after few epoches,
while the curves of the other two methods keep stable.
This is an obvious over-fitting phenomenon,
which indicates that when using a small networks,
AdvTraining seems to just memorize the adversarial samples, rather than learning the robust features.

\subsection{Ablation study}

\begin{figure}[!thb]
	\centering
	\scalebox{0.65}{\includegraphics{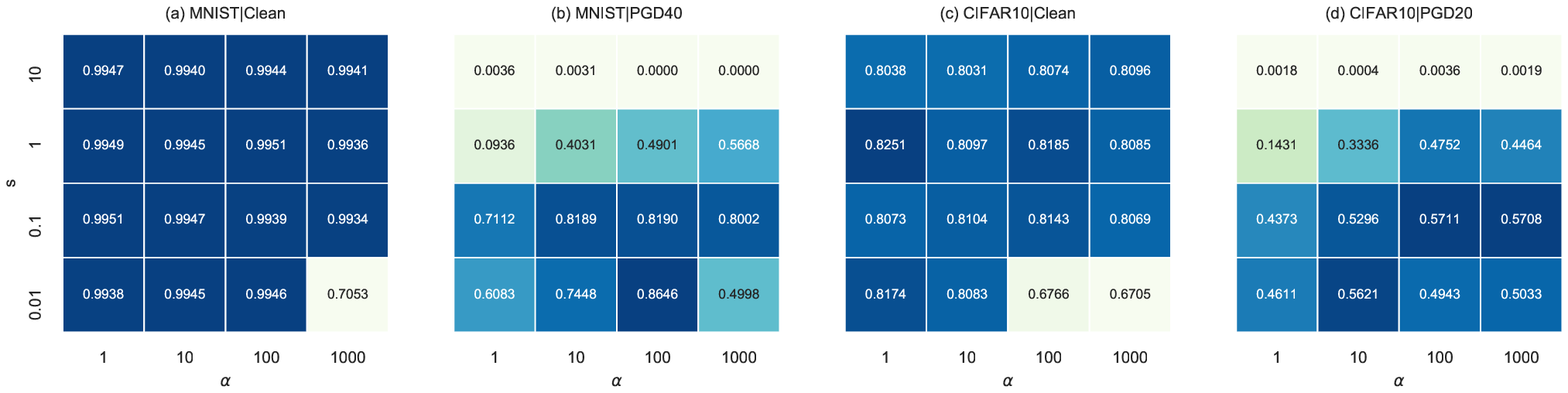}}
	\caption{ Classification accuracy on clean and adversarial samples with various $\alpha $ and $s $.}
	\label{ablation}
\end{figure}

In this section, we  aim to test the performance of our method under various configurations.
We apply PGD40 attack on MNIST dataset and PGD20 attack on CIFAR-10 dataset.

We first investigate the effects of hyperparameters $ \alpha  $ and $ s$ of \eqref{weightloss}.
As shown in Figure \ref{ablation},  different hyperparameters have little impact on classification accuracy on clean data,
while for adversarial examples, the selection of hyperparameters has a great impact on the accuracy.

\begin{figure}[!thb]
	\centering
	\scalebox{0.65}{\includegraphics{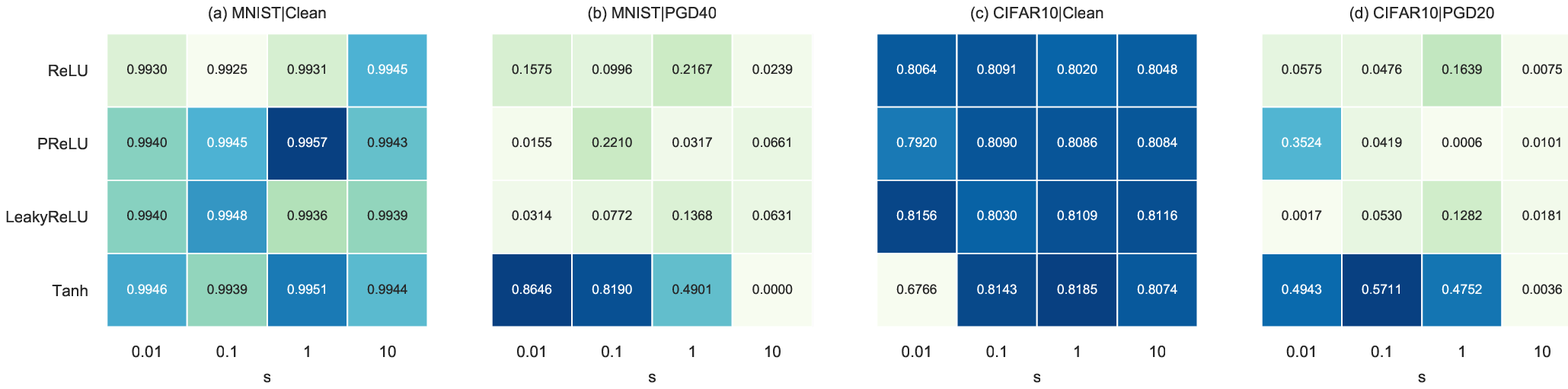}}
	\caption{Classification accuracy on clean and adversarial samples with various activation choices and $ s $ ($ \alpha = 100 $).}
	\label{ablation_act}
\end{figure}

\begin{figure}[!thb]
	\centering
	\scalebox{0.68}{\includegraphics{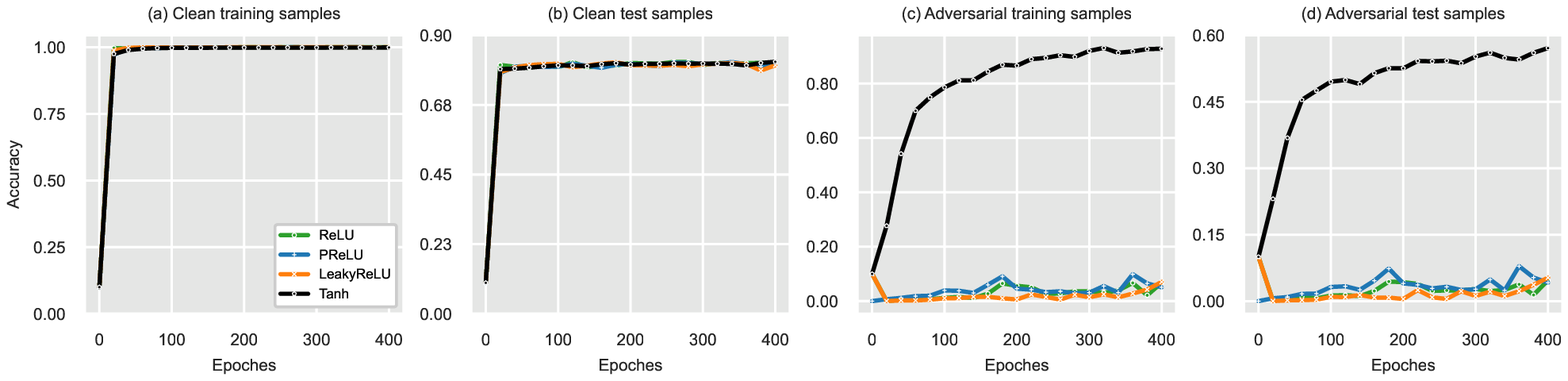}}
	\caption{Comparison of training and test accuracy using different activation functions on  CIFAR-10 dataset ($ \alpha = 100 $ and $ s=0.1 $).}
	\label{ablation_cifar10}
\end{figure}

Next, to validate our findings in Section 3.2,
we examine the performance of our method using  different activation functions,
including ReLU, PReLU, LeakyReLU and Tanh.
We fix $\alpha=100$ and allow $ s $ to be variable.
As reflected in Figures \ref{ablation_act} and  \ref{ablation_cifar10},
on clean data,  there is almost no difference in the prediction accuracy for different activation choices.
But, for adversarial examples,  using Tanh in the activation unit achieves best performance.

\section{Concluding Remarks}
It is a meaningful direction to improve the robustness of neural network
by improving the classification part of the network \cite{mustafa2019, pang2018, pang2020, wang2018}.
This direction could overcome some shortcomings of commonly used adversarial training approaches \cite{madry2018, shafahi2019, wang2020},
such as heavy computation load, and robustness derived from the sacrifice of accuracy on clean data.

Along this line, in this paper,
we first proved that the robustness of neural network is equivalent to certain margin condition between the classifier weights.
Further, from two different perspectives,
we analyzed why non-saturated function such as Tanh is a better activation choice than commonly used ReLU type function under the framework like ours.
Based on these findings, we developed a novel light-weight-penalized defensive method,
which has a clear geometric background  and can be easily used in various existing neural network architectures.
The experimental results clearly showed that, in most cases,
the proposed method outperformed the baseline defenses in terms of  both robustness and prediction accuracy.

There still leave some problems to be further investigated.
In Section 3.2, we analyze the relationship between activation choice and robustness intuitively,
while rigorous theoretical analysis deserves further study.
Moreover, the paper focuses on the improvement  of the classification part,
and whether it is possible to apply the developed ideas in improving the nonlinear encoder part is also a promising problem.

\newpage

\appendix
\renewcommand{\appendixname}{Appendix}

\section{}
\begin{lmm}\label{lmm1}
Let $M \in \mathbb{R}^{n \times n}$ be a symmetric matrix  such that $ M_{i,i}=a $ and $ M_{i,j}=b $,
$ 1\leq i,j \leq n $, $ j\neq i $.
Then its determinant satisfies
\begin{align}
\label{a1}
\det (M) = (a + bn -b) (a - b)^{n-1}.
\end{align}	
\end{lmm}

\begin{proof}
The desired result (\ref{a1}) could be obtained  by applying an inductive argument to the dimension $ n $.
\end{proof}

\begin{lmm}
\label{lmm2}
For $W \in \mathbb{R}^{K \times P}$,
denote by $w_i $ the $ i $-th row of the matrix.
Suppose that  $ 1<K \leq P + 1 $ and
\begin{align*}
  w_i^T w_j =
                 \left \{
                         \begin{array}{ll}
                                   1,      & i=j \\
                                   \rho,  &  j \neq i
                         \end{array}
                 \right.,
\end{align*}
where $\rho $ is some constant. Then the following relation
\begin{align}
\label{R1}
	\rho \geq \frac{1}{1-K}
\end{align}
holds true. Moreover, there exists some $ W $ that makes the equality hold.

\end{lmm}

\begin{proof}
Set  $\Sigma=WW^T $.
It is easy to verify that
\begin{align*}
\Sigma_{ij} =
\left \{
\begin{array}{ll}
1, & i=j \\
\rho, & i \not = j.
\end{array} \right.
\end{align*}
Let  $ \lambda $ be an arbitrary eigenvalue of $ \Sigma $.
Then using Lemma \ref{lmm1}, we have
\begin{align*}
0=\det (\lambda I - \Sigma) = (\lambda-1 - \rho K + \rho) (\lambda -1 +\rho)^{K -1},
\end{align*}
which means that
$ \lambda=1+ \rho K - \rho $ or $ \lambda =1- \rho $.
Since $\Sigma$ is positive semidefinte, then $ \lambda \geq 0 $.
Therefore,
\begin{align*}
	\frac{1}{1-K} \leq \rho \leq 1.
\end{align*}
which gives the desired inequality \eqref{R1}.

Next, we shall construct a special $ W $ from $ \Sigma $ to make the equality in  \eqref{R1} holds true.
First, we set  $\rho=\frac{1}{1-K}$ in $ \Sigma $.
It is easy to verify that the rank of $ \Sigma $ is at most $K-1$.
Since $ \Sigma $ is symmetric,
then  there must exist a matrix $ V \in \mathbb{R}^{K \times K-1}$ such that $\Sigma = V V^T$.
Since $ P\geq K-1 $, we define $ W=[V,\mathbf{0}] \in \mathbb{R}^{K \times P}$.
It is easy to verify that this $ W$ satisfies
\begin{align*}
 w_i^T w_j =
\left \{
\begin{array}{ll}
1, & i=j, \\
\frac{1}{1-K}, & i \not = j.
\end{array} \right.
\end{align*}
\end{proof}

\textbf{The proof of Theorem \ref{margin}:}
\begin{proof}
For brevity of statement, we assume that $ \|w_i \|=1 $, $ 1 \leq i \leq K $.
Note that  \eqref{ww} achieves its maximum
if and only if there exists a solution to the following optimization problem
	\begin{align}
    \label{inner}
	     \min_{W} \max_{1 \leq i,j \leq K \atop j \neq i} w_i^T w_j.
	\end{align}

First, we are to demonstrate that  the optimal solution of (\ref{inner})  satisfies  the condition that all $ \{w_i^T w_j \} $, $i \not= j$ are equal.
If this is not true,  there must exist a optimal solution  $ \tilde{W} $ which violates the equal condition, such that
\begin{align*}
 \max_{1 \leq i,j \leq K \atop j \neq i} \tilde{w}_i^T \tilde{w}_j
 =&\min_{W } \max_{1 \leq i,j \leq K \atop j \neq i} w_i^T w_j
 \\[5pt]
 \leq  &  \min_{W \cap \{ w_i^T w_j=\rho\} }\max_{1 \leq i,j \leq K \atop j \neq i} w_i^T w_j=\frac{1}{1-K},
 \quad K>1,
\end{align*}
where the last equality follows from Lemma \ref{lmm2}.
Without loss of generality, we assume that
\begin{align}
\label{maxs}
	     \tilde{w}_1^T \tilde{w}_2
    = \max_{1 \leq i,j \leq K \atop j \neq i} \tilde{w}_i^T \tilde{w}_j \leq \frac{1}{1-K}<0.
\end{align}
Since the inner products of the rows of  $ \tilde{W} $ are not all equal,
there must  exist a row vector, denoted as $\tilde{w}_3 $, such that
\begin{align*}
	     \tilde{w}_1^T \tilde{w}_3 < \tilde{w}_1^T \tilde{w}_2.
\end{align*}

Now, we construct an auxiliary vector $\tilde{w}(t) $ from $ \tilde{w}_1 $ and $ \tilde{w}_2 $   as follows
\begin{align*}
	\tilde{w}(t) = \frac{t  \tilde{w}_1 + (1-t) \tilde{w}_3}{\|t  \tilde{w}_1 + (1-t) \tilde{w}_3\|_2}, \quad t \in [0, 1].
\end{align*}
Note that $\tilde{w}_1^T \tilde{w}_3 < \tilde{w}_1^T \tilde{w}_2<0 $.
Then we have
   \begin{align*}
     \|t  \tilde{w}_1 + (1-t) \tilde{w}_3\|_2^2 < 1,
     \quad
     \forall t \in (0, 1).
    \end{align*}
By \eqref{maxs}, it is obvious that
   \begin{align*}
	 (t  \tilde{w}_1 + (1-t) \tilde{w}_3)^T \tilde{w}_j  \leq  {\tilde{w}_1}^T \tilde{w}_2, \quad \forall j \neq 1,3.
	\end{align*}
Using the above two estimates, we have
	\begin{align}
    \label{A5}
    \begin{split}
	\tilde{w}^T(t) \tilde{w}_j
     = &\frac{(t  \tilde{w}_1 + (1-t) \tilde{w}_3)^T\tilde{w}_j}{\|t  \tilde{w}_1 + (1-t) \tilde{w}_3\|_2}
        \\[5pt]
     < &\frac{(t  \tilde{w}_1 + (1-t) \tilde{w}_3)^T\tilde{w}_j}{1}
            \leq  {\tilde{w}_1}^T \tilde{w}_2,
     \quad  j \neq 1,3.
     \end{split}
	\end{align}
Next, note that $\tilde{w}^T(1) \tilde{w}_3={\tilde{w}_1}^T \tilde{w}_3 <  \tilde{w}_1^T \tilde{w}_2$.
By properties of continuous functions, we know there must exist a small $ \epsilon>0 $ such that
for any $ t^* \in (1-\epsilon, 1) $,  $ \tilde{w}^T(t^*) \tilde{w}_3 <  \tilde{w}_1^T \tilde{w}_2 $ holds true.
   Combining this estimate with  \eqref{A5}, we have
 \begin{align*}
     	\tilde{w}^T(t^*) \tilde{w}_j
    < \tilde{w}_1^T  \tilde{w}_2,
    \quad   j \neq 1.
\end{align*}
Set $ \hat{w}_1= \tilde{w}(t^*) $. It is obvious that
     \begin{align*}
	    \max_{1 \leq j \leq K \atop j \neq 1} \hat{w}_1^T \tilde{w}_j
      <\tilde{w}_1^T  \tilde{w}_2, \quad j \neq 1.
	\end{align*}
Taking a similar argument, we can construct a series of $ \hat{w}_i $, $ i=1,2,\cdots, K $ step by step,
which satisfies
     \begin{align}
     \label{A6}
	    \max_{1 \leq i,j \leq K \atop j \neq i} \hat{w}_i^T \hat{w}_j
      <\tilde{w}_1^T  \tilde{w}_2
     =\max_{1 \leq i,j \leq K \atop j \neq i} \tilde{w}_i^T \tilde{w}_j.
	\end{align}
However, \eqref{A6} contradicts the fact that $ \tilde{W} $ is the optimal solution to problem \eqref{inner}.
Thus,  the optimal solution of (\ref{inner}) must satisfy the condition that all $ w_i^T w_j  $, $i \not= j$ are equal.
Finally,  this property with Lemma \ref{lmm2} completes the proof of the theorem.
\end{proof}

\end{document}